\documentclass[letterpaper]{article}
\usepackage{aaai25} 
\usepackage{times} 
\usepackage{helvet} 
\usepackage{courier} 
\usepackage[hyphens]{url} 
\usepackage{graphicx} 
\usepackage{graphicx} 
\usepackage{natbib} 
\usepackage{caption} 
\usepackage{amssymb}
\usepackage{amsmath}
\usepackage{booktabs}
\usepackage{adjustbox}
\usepackage{multirow}
\usepackage{amsthm}
\usepackage{colortbl}
\usepackage{xcolor}

\definecolor{first}{RGB}{135,219,135}
\definecolor{second}{RGB}{177,232,177}
\definecolor{third}{RGB}{219,244,219}

\newcommand{\R}[0]{\ensuremath{\mathbb{R}}}
\newcommand{\N}[0]{\ensuremath{\mathbb{N}}}
\newcommand{\Dc}[0]{\ensuremath{\mathcal{D}}}

\newcommand{\param}[0]{\ensuremath{\Delta}}
\newcommand{\pre}[0]{\ensuremath{\mathrm{pre}}}

\newcommand{\add}[0]{\ensuremath{\mathrm{add}}}
\newcommand{\del}[0]{\ensuremath{\mathrm{del}}}

\newcommand{\specificity}[0]{\ensuremath{\mathrm{spec}}}
\newcommand{\actionSchemas}[0]{\ensuremath{\mathcal{A}}}
\newcommand{\predicates}[0]{\ensuremath{\mathcal{P}}}
\newcommand{\objects}[0]{\ensuremath{\mathcal{O}}}
\newcommand{\initialState}[0]{\ensuremath{s_0}}
\newcommand{\goal}[0]{\ensuremath{G}}
\newcommand{\none}[0]{\ensuremath{\mathrm{None}}}
\newcommand{\partialSpaceSearch}[0]{\ensuremath{\mathrm{PS}^2}}
\newcommand{\stateSpaceSearch}[0]{\ensuremath{\mathrm{S}^3}}
\newcommand{\hFF}[0]{\ensuremath{h^{\mathrm{FF}}}}

\newcommand{\ilg}[0]{\ensuremath{\mathrm{ILG}}}
\newcommand{\actions}[0]{\ensuremath{\mathbb{A}}}
\newcommand{\states}[0]{\ensuremath{\mathbb{S}}}
\newcommand{\restrictionHeuristic}[0]{\ensuremath{h_\mathrm{rs}}}
\newcommand{\restrictedFF}[0]{\ensuremath{\hFF_\mathrm{rs}}}
\newcommand{\actionSet}[0]{\ensuremath{\Lambda}}
\newcommand{\uninstantiated}[0]{\ensuremath{\text{\underline{\hspace{0.5em}}}}}

\newcommand{\unavoidableAdd}{\ensuremath{\mathrm{unav^+}}}
\newcommand{\unavoidableDel}{\ensuremath{\mathrm{unav^-}}}
\newcommand{\optionalAdd}{\ensuremath{\mathrm{opt^+}}}
\newcommand{\optionalDel}{\ensuremath{\mathrm{opt^-}}}

\newtheorem{definition}{Definition}

\newtheorem{theorem}{Theorem}

\newcommand{\newtext}[2]{\textcolor{red}{#1}}

\usepackage[switch]{lineno}

\usepackage{color}

\title{Leveraging Action Relational Structures for Integrated Learning and Planning}
\author{
    Ryan Xiao Wang,
    Felipe Trevizan
}
\affiliations{
    School of Computing, Australian National University\\
    \{ryan.wang, felipe.trevizan\}@anu.edu.au
}

\begin{document}

\maketitle

\begin{abstract}
    Recent advances in planning have explored using learning methods to help
    planning. However, little attention has been given to adapting search
    algorithms to work better with learning systems. In this paper, we introduce
    partial-space search, a new search space for classical planning that
    leverages the relational structure of actions given by PDDL action schemas
    -- a structure overlooked by traditional planning approaches. Partial-space
    search provides a more granular view of the search space and allows earlier
    pruning of poor actions compared to state-space search. To guide
    partial-space search, we introduce action set heuristics that evaluate sets
    of actions in a state. We describe how to automatically convert existing
    heuristics into action set heuristics. We also train action set heuristics
    from scratch using large training datasets from partial-space search. Our
    new planner, LazyLifted, exploits our better integrated search and learning
    heuristics and outperforms the state-of-the-art ML-based heuristic on IPC
    2023 learning track (LT) benchmarks. We also show the efficiency of
    LazyLifted on high-branching factor tasks and show that it surpasses LAMA in
    the combined IPC 2023 LT and high-branching factor benchmarks.
\end{abstract}

\section{Introduction}\label{sec:introduction}

Development in machine learning has led to significant interest in learning for
planning in recent years. The focus of existing work has been on learning domain
knowledge automatically in a domain-independent manner to help planning systems.
Such domain knowledge can be in the form of generalised policies
\citep{toyer-et-al-aaai2018, toyer-et-al-jair2020, stahlberg-et-al-icaps2022, stahlberg-et-al-kr2022,
    stahlberg-et-al-kr2023}, subgoal structures \citep{drexler-et-al-jair2024,
    bonet-geffner-jair2024}, and heuristics that guide search algorithms. We focus
on the common and successful approach of learning heuristics. Existing works
have learned heuristics using neural networks
\citep{shen-et-al-icaps2020,karia-srivastava-aaai2021,chen-et-al-aaai2024}, as
well as classical machine learning and optimisation methods
\citep{frances-et-al-ijcai2019,chen-et-al-icaps2024}, and have explored learning
heuristics to estimate metrics other than cost-to-goal
\citep{ferber-et-al-icaps2022,chrestien-et-al-neurips2023,hao-et-al-ijcai2024}.
However, to our knowledge, no work has explored how search algorithms can be
adapted to work better with learned heuristics. Specifically, all previously
mentioned learning for planning approaches use state-space search, where
heuristics are functions of planning states. Little discussion has been given to
if this state-based interface between heuristics and search is the most
effective, particularly for learning-based heuristics.

In this paper, we introduce partial-space search, a more granular search space
than state-space search that takes into account the actions schemas given
by the PDDL domain description \citep{haslum-et-al-2019}. Action schemas define
a natural tree structure of the space of partially instantiated actions, which
we call partial actions. Partial-space search is a refinement of state space
search that explores this tree structure for more granularity. Searching in this
space represents instantiating action schemas one parameter at a time, and
transitioning between states when all parameters are instantiated. Partial-space
search also represents a change in the interface between heuristics and search,
where heuristics are functions of state and partial action pairs.

Partial-space search offers both general advantages and learning specific
advantages. The added granularity makes heuristic search more efficient by
dividing a single state expansion in state-space search into multiple smaller
expansion steps with lower branching factors. This is particularly useful for
tasks with high branching factors, regardless of if learned heuristics are used.
For learning, we show this additional granularity allows the generation of
larger training datasets from the same training tasks and plans, which can be
used to train more powerful heuristics. Moreover, designing heuristics require a
trade-off between evaluation speed and heuristic informedness. As we will
discuss, partial-space search shifts this trade-off towards favouring
informedness. This benefits learning-based heuristics, in particular as they
continue to become more accurate and potentially slower to evaluate. Ultimately,
these advantages mean that partial-space search is potentially better suited for
future learning-based heuristics.

To construct heuristics for partial-space search, we view partial actions as
sets of actions, and define action set heuristics as functions of state and
action sets pairs. To make partial-space search compatible with any existing
state space heuristics, we show how to translate them into action set
heuristics. Since our focus is on learning-based heuristics, we also introduce
two novel graph representations for state and action set pairs, and extend the
methods proposed in
\citet{chen-et-al-icaps2024,hao-et-al-ijcai2024,chen-thiebaux-neurips2024} to
learn action set heuristics from our graph representations.

To evaluate our approach, we implement our contributions in a new planner called
LazyLifted. We use the International Planning Competition 2023 learning track
benchmarks and additional high branching factor benchmarks. When using the
translated action set heuristics (specifically, $\hFF$), partial-space search
outperforms state-space search with the original heuristic under high branching
factors. When using the learned action set heuristics, LazyLifted outperforms
the state-of-the-art learned heuristic. Altogether, LazyLifted outperforms LAMA
on the combined benchmarks in terms of coverage.

\section{Background}\label{sec:background}

\paragraph{Planning} A classical planning problem is a pair $\Pi = \langle D, I
    \rangle$ of a \emph{domain} $D$ and \emph{instance} $I$
\citep{geffner-bonet-2013, haslum-et-al-2019}. The domain $D = \langle
    \predicates, \actionSchemas \rangle$ provides the high level structure of
the problem by defining a set of \emph{predicates} $\predicates$ and a set
of \emph{action schemas} $\actionSchemas$. The instance $I = \langle
    \objects, \initialState, \goal \rangle$ provides the task specific
information by defining a set of \emph{objects} $\objects$, the
\emph{initial state} $\initialState$, and the \emph{goal} $\goal$. Each
predicate $P \in \predicates$ can be instantiated with objects $o_1, \ldots,
    o_k \in \objects$ to form a ground \emph{atom} $P(o_1, \ldots, o_k)$, where
the arity $k$ depends on $P$. Each action schema $A \in \actionSchemas$ has
\emph{schema arguments} $\param(A)$, \emph{preconditions} $\pre(A)$,
\emph{add effects} $\add(A)$, and \emph{delete effects} $\del(A)$.
Preconditions, add effects, and delete effects are expressed as predicates
instantiated with $\param(A)$. An action schema can be instantiated with
objects, resulting in a ground action where the schema arguments are
replaced with the instantiating objects. We denote the set of ground actions
with $\actions$. A predicate is \emph{static} if it does not appear in
the effect of any action schemas.

The problem $\Pi$ encodes a finite \emph{state space} $\states$, where each
state is a set of ground atoms, representing those that are true in the state.
The initial state $\initialState$ is an element of $\states$. The goal $\goal$
is a set of ground atoms, and a state $s$ is a \emph{goal state} if $\goal
    \subseteq s$. The ground actions $\actions$ define a transition system on
$\states$, where each action $a \in \actions$ can be applied in any state $s$
with $\pre(a) \subseteq s$, resulting in a new state $(s \setminus \del(a)) \cup
    \add(a)$. The set of applicable actions in a state $s$ is denoted $\actions_s$.
The solution to a planning problem is a \emph{plan} -- a finite sequence of
ground actions that can be applied sequentially starting from the initial state,
whose sequential application leads to a goal state. We assume all actions have
unit cost, and the cost of a plan is simply its length.

\paragraph{State-space search} The most prevalent approach to classical planning
is heuristic search on the state space. A (state space) \emph{heuristic} $h:
    \states \rightarrow \R_{\ge 0}$ estimates the quality of a state, with lower
values being better. A search algorithm, such as A* or Greedy Best First Search
(GBFS), explores the state space under the guidance of the heuristic. This
\emph{state-space search} typically requires grounding the planning task, i.e.,
computing the set of all reachable ground atoms and actions
\citep{helmert-jair2006}. Tasks where this is hard or infeasible are called
\emph{hard to ground}. Lifted planners, that do not require grounding the entire
task, have been developed to address this issue \citep{correa-et-al-icaps2020}.
Computation of various heuristics for lifted planners can be harder than for
grounded planners, but recent works have shown that well-known heuristics can be
computed efficiently \citep{correa-et-al-icaps2021,correa-et-al-aaai2022}.

\paragraph{Learning for Planning via Graphs}

\newcommand{\objectColour}{\ensuremath{\mathrm{ob}}}
\newcommand{\achievedGoal}{\ensuremath{\mathrm{ag}}}
\newcommand{\unachievedGoal}{\ensuremath{\mathrm{ug}}}
\newcommand{\achievedNonGoal}{\ensuremath{\mathrm{ap}}}

The current state-of-the-art method for learning heuristics uses the
Weisfeiler-Lehman (WL) algorithm. The WL algorithm is an iterative algorithm
that computes feature vectors representing the frequency of substructures in a
graph. The number of iteration is a hyperparameter that determines the
complexity of the substructures. \citet{chen-et-al-icaps2024} used the WL
algorithm to map graph representations to feature vectors and then learn a
linear function of these features. The graph representation they used is the
\emph{Instance Learning Graph} (ILG). Given a state $s$ and a planning problem
$\Pi$, the ILG is a graph $G = \langle V, E, c, l \rangle$ where
\begin{itemize}
    \item Vertices $V = \objects \cup s \cup \goal$.
    \item Edges $E = \bigcup_{p = P(o_1, \ldots, o_k) \in s \cup \goal}
              \{\langle p, o_1\rangle, \ldots, \langle p, o_k\rangle\}$.
    \item Vertex colours $c(v) \in (\{\achievedGoal, \achievedNonGoal,
              \unachievedGoal\} \times \predicates) \cup \{\objectColour\}$
          with
          \begin{align*}
              v \mapsto \begin{cases}
                            \objectColour,         & \text{if } v \in \objects                                \\
                            (\achievedGoal, P),    & \text{if } v = P(o_1, \ldots, o_k) \in s \cap \goal      \\
                            (\achievedNonGoal, P), & \text{if } v = P(o_1, \ldots, o_k) \in s \setminus \goal \\
                            (\unachievedGoal, P),  & \text{if } v = P(o_1, \ldots, o_k) \in \goal \setminus s
                        \end{cases}
          \end{align*}
    \item Edge labels $l(e) = i$ for $e = \langle p, o_i \rangle$.
\end{itemize}

States can be mapped to feature vectors by applying the WL algorithm to their
ILG representations. Given small example instances on the same domain with
example plans, \citet{chen-et-al-icaps2024} trained machine learning models
to map feature vectors of states in the example plans to their distance to
goal in the example plan. They developed the WL-GOOSE system, which uses
the trained models as state-of-the-art learned heuristics to guide GBFS in
state-space search. Note that in their work, all atoms whose predicates are
static are ignored when mapping states to their ILG.

\section{Partial-space Search}\label{sec:partial-space-search}

State-space search ignores the inherent hierarchy of decisions induced by action
schemas and their instantiation process. For example, in a planning task of
making and serving sandwiches, humans may first decide whether to make or serve
a sandwich, and then decide what type of sandwich or who to serve the sandwich
to. In state-space search, the planner would typically consider all the sandwich
making and sandwich serving actions at the same time, rather than deciding in a
hierarchical fashion. To exploit this hierarchy, we introduce partial actions.

\begin{definition}\label{def:partial-action} A \emph{partial action} is an
    action schema with some (incl. none and all) of its schema arguments
    instantiated by objects. It is denoted as $A(o, \uninstantiated)$ where $A
        \in \actionSchemas$, $o \in \objects$, and $\uninstantiated$ indicates an
    uninstantiated schema argument. The special partial action $\none$
    indicates that not even the action schema itself specified. For a partial
    action $\rho$ with $k$ schema arguments instantiated, its \emph{specificity}
    is $\specificity(\rho) = k+1$, where $\specificity(\none) = 0$.
\end{definition}

\begin{figure}
    \centering
    \includegraphics*[width=0.4742\textwidth]{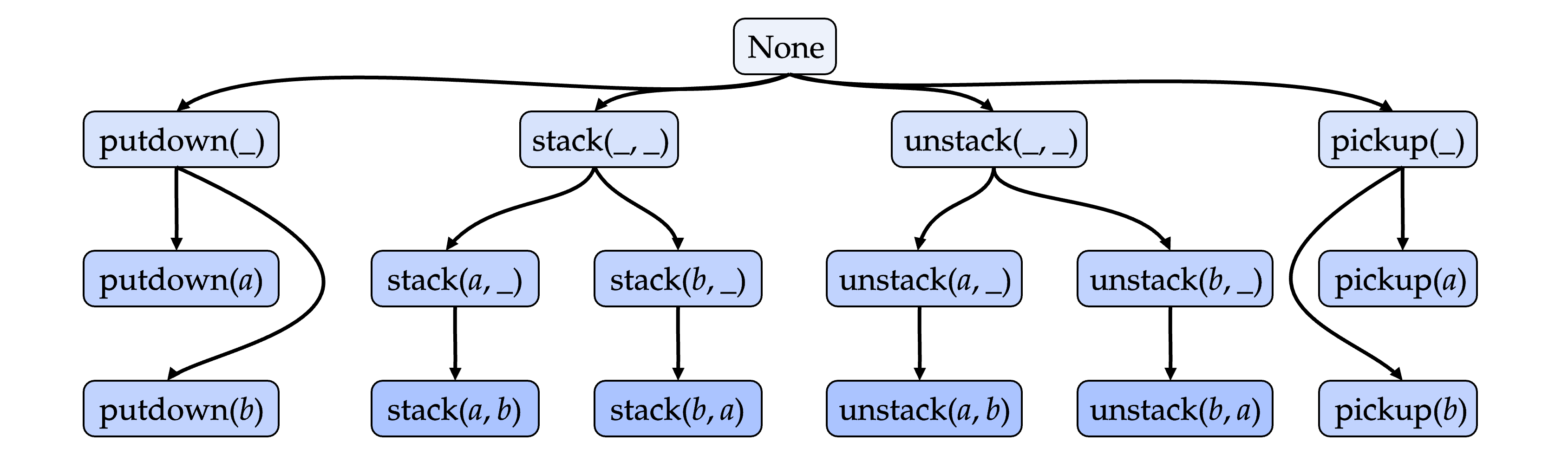}
    \caption{Partial action tree for a Blocksworld task with objects $a$ and
        $b$. Darker node colours indicate higher specificity.}
    \label{fig:partial-action-tree}
\end{figure}

We assume schema arguments of action schemas have a fixed order, and that
partial actions have a prefix of the ordered schema arguments instantiated. This
way, the partial actions form a \emph{partial action tree} rooted at $\none$,
whose children are the action schemas, and whose children are partial actions
with specificity $2$, etc. (see Figure \ref{fig:partial-action-tree}). A partial
action $\rho$ is a representation of the set of ground actions in its subtree,
denoted $\actions^\rho$. We use $\actions_s^\rho$ to denote the set
$\actions^\rho \cap \actions_s$. A partial action is applicable in a state $s$
if $\actions_s^\rho$ is not empty. Through the partial action tree, partial
actions naturally represent the hierarchical nature of actions, which allows
defining partial-space search.

\begin{definition}\label{def:partial-space-search} Given a classical planning
    task $\Pi$, the \emph{partial-space search} of $\Pi$ is a search problem
    whose search nodes have the form $\langle s, \rho \rangle$ where $s \in
        \states$ and $\rho$ is a partial action. The search starts at the root node
    $\langle s_0, \none\rangle$. The successor nodes of each node $\langle s,
        \rho \rangle$ is given by
    \begin{itemize}
        \item If $\rho$ is not fully instantiated, the successors are the nodes
              with state $s$ and partial action $\rho'$, where $\rho'$ is an
              applicable child of $\rho$ in the partial action tree.
        \item Otherwise, the successor is $\langle s', \none \rangle$, where $s'$
              is the resulting state of applying $\rho$ in $s$.
    \end{itemize}
\end{definition}

Partial-space search finds a plan for $\Pi$ by reaching a node $\langle s, \none
    \rangle$ where $s$ is a goal state. The plan is the sequence of fully
instantiated partial actions from the root to this node. Moreover, in
practice many search nodes have only one successor node. Such nodes are
repeatedly expanded till obtaining multiple successor nodes. The following
theorem states the correctness of partial-space search. \newtext{}{Its proof is in
    appendix.}

\begin{theorem}\label{thm:partial-space-search}
    Partial-space search is sound and complete for solving planning tasks.
\end{theorem}

\begin{proof}
    Given a planning task $\Pi$ and let $\pi$ be a plan found by partial-space
    search, we first show soundness, i.e., that $\pi$ is a valid plan.
    Partial-space search found $\pi$ by traversing a sequence of nodes $\langle
        s_0, \rho_0 \rangle, \langle s_1, \rho_1 \rangle, \ldots, \langle s_n,
        \rho_n \rangle$ where $s_0 = \initialState$, $\rho_n = \none$, and $s_n$ is
    a goal state. For each $i$, the definition of partial-space search
    guarantees that $\rho_i$ is an applicable partial action in $s_i$. Moreover,
    if $\rho_i$ is fully-instantiated then $s_{i+1}$ is the resulting state of
    applying $\rho_i$ in $s_i$, and otherwise the state is unchanged
        from one node to the next, i.e., $s_i = s_{i+1}$. Given these, since the
    plan $\pi$ is obtained by only keeping the fully-instantiated $\rho_i$'s, it
    must be a valid plan for $\Pi$.

    Next, we show completeness, i.e., that partial-space search can
        always find a plan for $\Pi$ as long as a valid plan exists. Let
    $\pi$ be such a valid plan for $\Pi$. Each action $a$ in $\pi$ can be
    decomposed into a sequence of partial actions with increasing
    specificity, starting at $\none$ and ending at $a$. Combining these
    partial actions with the state that $a$ was applied in yields a sequence
    of partial space search nodes. Doing this for all actions in $\pi$
    yields a sequence of partial space search nodes that are in the search
    tree of partial-space search. Therefore, partial-space search will find
    $\pi$.
\end{proof}

Partial-space search decomposes the process of expanding a single state in state
space search into multiple smaller expansion steps that are akin to gradually
narrowing down from the set of applicable actions. This can factorise the
branching factor of state-space search, as shown in the example in Figure
\ref{fig:expansion-tree}. Here, partial-space search removes the need to
evaluate all successor states resulting from the application of the $A_2$ action
schema because it determines $A_2$ itself to be a bad choice. This way, partial
space search achieves the same outcome as state-space search but with fewer
heuristic evaluations, which is particularly useful when branching factors are
high or the heuristic is expensive to evaluate but very accurate. The latter
case is what we mean by partial-space search moving the trade-off between
heuristic evaluation speed and informedness towards informedness. This shift is
particularly relevant for learning-based heuristics. Machine learning,
particularly deep learning, has trended towards more powerful models that are
more expensive to compute. Partial-space search is more suited for this trend.
Beyond being more suited for applying learning-based heuristics, partial-space
search also allows generating larger training datasets due to the increased
granularity. We discuss this further once we introduce how we generate training
data using partial-space search.

\begin{figure}
    \centering
    \includegraphics[width=0.4742\textwidth]{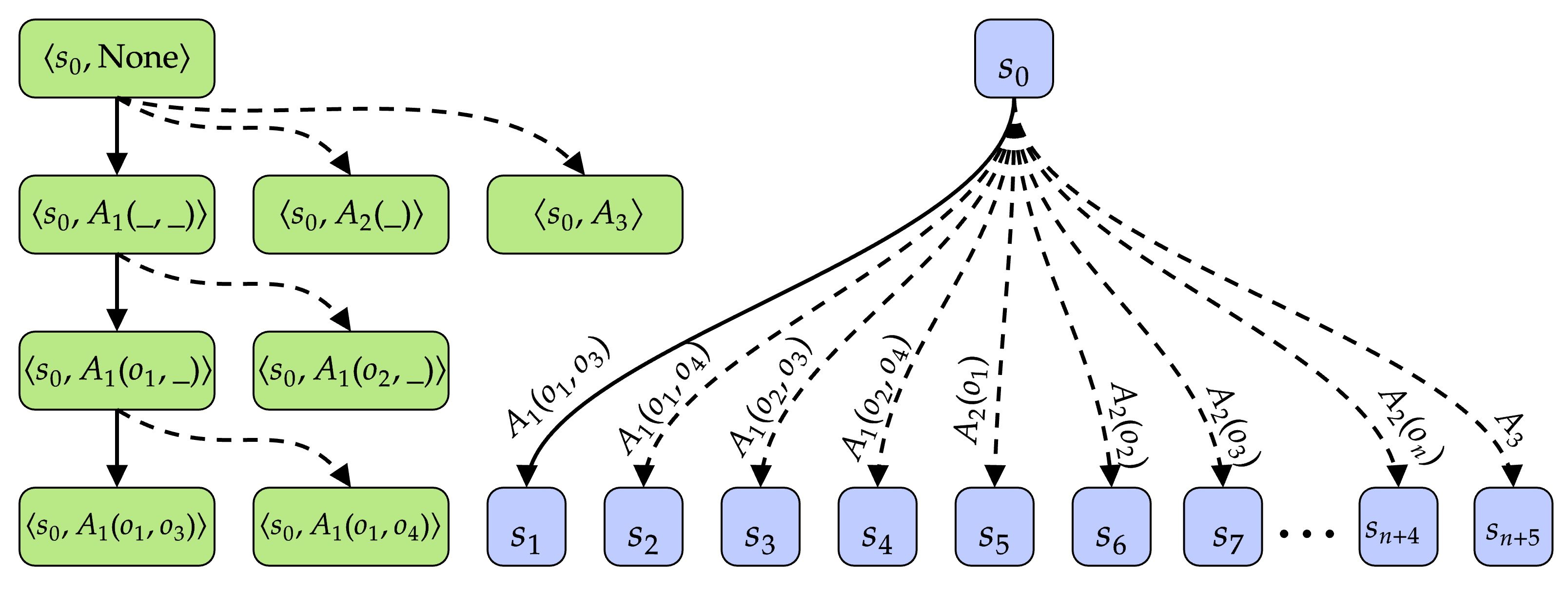}
    \caption{Example expansion trees of partial-space search (left) and state
        space search (right). Solid lines indicate expanded nodes, dashed lines
        indicate unexpanded nodes. Expansion decisions are guided by an informed
        heuristic.}
    \label{fig:expansion-tree}
\end{figure}

\section{Action Set Heuristics}

In order to guide partial-space search, we need heuristics that evaluate search
nodes of partial-space search, which are state and partial action pairs. Given
such a pair $\langle s, \rho \rangle$, we define heuristics by viewing $\rho$ as
the set of applicable actions $\actions^\rho_s$. This is a more general form of
heuristics than those we need, but helps simplify our discussion.

\begin{definition}
    An \emph{action set heuristic} is a function $h: \states \times 2^\actions
        \rightarrow \R$. Given a partial action $\rho$, we will use $h(s, \rho)$
    to denote $h(s, \actions_s^\rho)$.
\end{definition}

For a classical planning task, action set heuristics can be used
for heuristic search over the partial state space. Moreover, given an action
set heuristic $h$, we can obtain a state space heuristic $h'$ with $s \mapsto
    h(s, \none)$. Next, we describe two ways to obtain action set heuristics.

\subsection{Restriction Heuristics}

Our first method is to translate existing state space heuristics to action set
heuristics. Given a state space heuristic $h$, a state $s$, and a set of actions
$\actionSet$, we obtain an action set heuristic by modifying the original
problem such that only actions in $\actionSet$ can be applied in $s$, then
computing $h(s)$ in this modified problem. We achieve this by adding a predicate
$\epsilon$ to the problem to indicate if an action in $\actionSet$ has been
applied. We then modify the actions such that all actions add $\epsilon$ and
only actions in $\actionSet$ can be applied without $\epsilon$. More formally:

\begin{definition}[Restriction heuristic]\label{def:restriction-heuristic}
    For a planning task $\Pi$, given a state space heuristic $h$, a state $s$,
    and a set of actions $\actionSet$, the \emph{$\actionSet$-restricted task}
    $\Pi_\actionSet$ is $\Pi$ with the following modifications applied
    sequentially:
    \begin{enumerate}
        \item an additional predicate $\epsilon$, which has no parameters and is
              initially false, is added;
        \item the predicate $\epsilon$ is added as a precondition to all action
              schemas;
        \item each $a \in \actionSet$ is added as an action schema with no
              schema arguments that is identical to $a$, except with the
              additional add effect $\epsilon$;
    \end{enumerate}
    The \emph{restriction heuristic} $\restrictionHeuristic(s, \actionSet)$ is
    the action set heuristic that takes the value of $h(s)$ computed for the
    task $\Pi_\actionSet$.
\end{definition}

\noindent Despite its generality,
computing restrictions heuristics by definition can be computationally
expensive. This is because the task $\Pi_\actionSet$ must be constructed for
each possible action set $\actionSet$, which can be exponential in the number of
actions. However, for some heuristics, it is possible to overcome this by not
explicitly computing the restricted task.

An example heuristic whose restricted version can be efficiently computed is the
lifted version of the $\hFF$ heuristic \citep{correa-et-al-aaai2022}. Here, we
only outline the necessary adaptations for computing its restriction heuristic
and more details can be found in the appendix. \citet{correa-et-al-aaai2022}
utilised Datalog rules corresponding to action schemas, and required
preprocessing the task to generate these rules. For the restriction heuristic of
$\hFF$, namely $\restrictedFF$, we apply modifications 1 and 2
(Def.~\ref{def:restriction-heuristic}) to the task, which are independent of the
action set $\actionSet$, and then preprocess the modified task as described in
their work. After, for each heuristic evaluation, we apply modification 3 by
generating temporary Datalog rules corresponding to each action in $\actionSet$,
and then perform the heuristic computation as in their work. This modification
is temporary, i.e., it is discarded after the heuristic evaluation.

\subsection{Graph Representations}

Restriction heuristics represent a general approach to obtaining action set
heuristics from existing state space heuristics. An alternative is to learn
action set heuristics from scratch, which yields heuristics dedicated for
partial-space search instead of merely adapted state space heuristics. Recent
works on learning heuristics for planning have not only shown that very powerful
heuristics can be learned, but also that heuristic learning is very flexible. To
learn a heuristic using the WL algorithm, one only needs to define a graph
representation for search nodes and then specify through training dataset set
what properties the heuristic should have. There is no need to define the
heuristic function itself, and even better, the learned WL heuristics can also
be interpreted \citep{chen-et-al-icaps2024}.

To learn action set heuristics, we start by defining two novel graph
representations of state and action set pairs. These two graphs represent
shallow and deep embeddings of the action set into the ILG of the state,
respectively. They are used to generate feature vectors of state and action
pairs through the WL algorithm.

As a minor improvement over the ILG, we encode in the object colours the static
predicates with arity 1 that apply to each object. That is, like the ILG, both
of graphs will include object nodes. We use $\predicates_o$ to denote the set of
static predicates of arity 1 such that for each $P \in \predicates_o$, $P(o)$
holds in the initial state (and hence any state reachable from it). The colour
of the object node $o$ in our graphs is $\predicates_o$. This allows our graphs
to consider some static predicates, unlike the ILG that ignores all of them.

\paragraph{Action-Object-Atom Graph}

Our first graph, called the Action-Object-Atom Graph (AOAG, Def.
\ref{def:aoag}), is an extension of the ILG that directly encodes the actions in
the action set as nodes in the graph.

\begin{definition}[AOAG]\label{def:aoag}
    Given a state $s$ and a set of actions $\actionSet$, if $\actions_s
        \subseteq \actionSet$, then the \textbf{Action-Object-Atom Graph} (AOAG) of $(s,
        \actionSet)$ is simply the ILG of $s$ with static object colours. This
    in particularly happens when the partial action inducing $\actionSet$ is
    $\none$. Moreover, if $\actionSet = \{a\}$, then the AOAG is the ILG of
    the state obtained by applying $a$ in $s$, again with static object
    colours.

    Otherwise, let the ILG of $s$ with static object colours be $\langle V_\ilg,
        E_\ilg, c_\ilg, l_\ilg \rangle$, the AOAG is $\langle V, E, c, l \rangle$,
    where
    \begin{itemize}
        \item $V = V_\ilg \cup \actionSet$
        \item $E = E_\ilg \cup \bigcup_{a = A(o_1, \ldots, o_k) \in \actionSet}
                  \{\langle a, o_1 \rangle, \ldots, \langle a, o_k \rangle \}$
        \item $c(u)$ is $c_\ilg(u)$ if $u \not\in \actionSet$, otherwise
              for $u = A(\ldots) \in \actionSet$ we have $c(u) = A$.
        \item $l(e)$ is $l_\ilg(e)$ if $e \in E_\ilg$, otherwise for $e = \langle a,
                  o_i \rangle$ we have $l(e) = i$.
    \end{itemize}
\end{definition}

AOAG is a shallow embedding of the action set into the ILG of the state as it
only introduces vertices representing actions and edges connecting actions to
their parameters without representing the actions' effects. An example of AOAG
for Blocksworld is shown in Figure \ref{fig:aoag}.

\begin{figure}
    \centering
    \includegraphics[width=0.4742\textwidth]{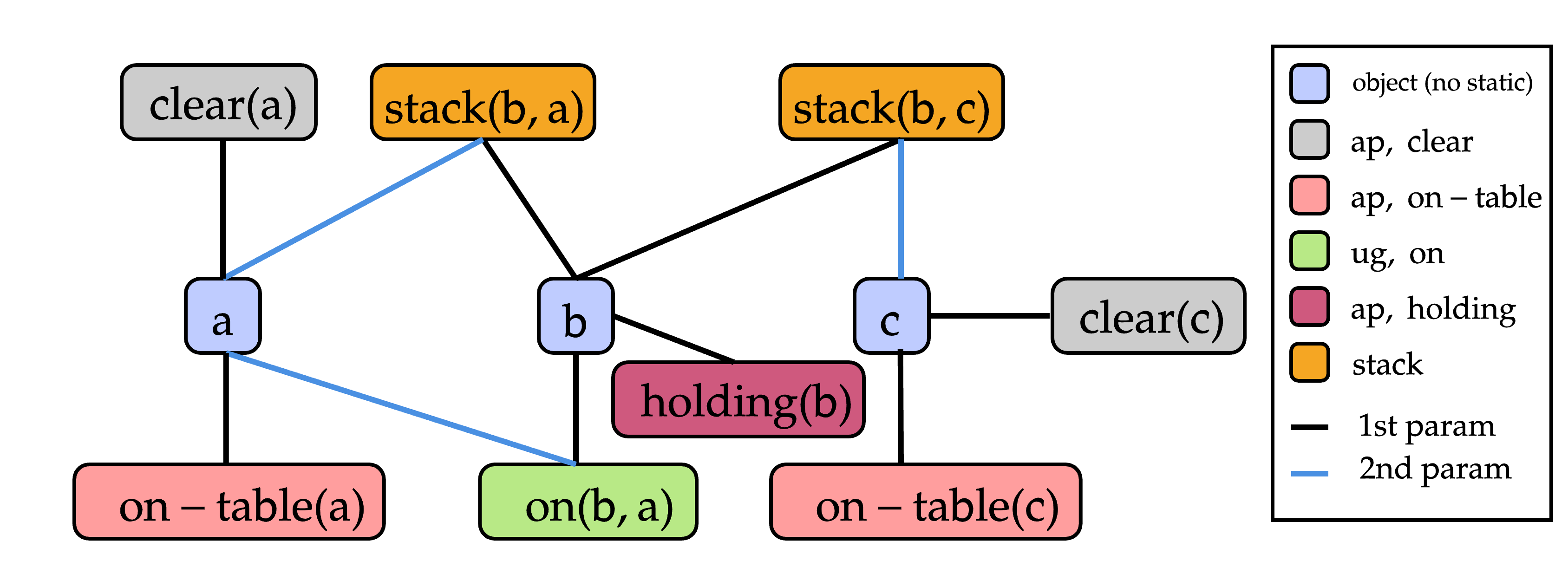}
    \caption{AOAG (Def. \ref{def:aoag}) of a Blocksworld Instance with blocks
        $a$ and $c$ on the table and $b$ being held, here the goal is to place
        $b$ on $a$ and $\actionSet$ contains all applicable $\mathtt{stack}$
        actions.}
    \label{fig:aoag}
\end{figure}

\paragraph{Action Effect Graph}

An alternative to directly encoding the actions in the graph is to instead
analyse the effects of the actions in $\actionSet$. This is the approach taken
to construct the Action Effect Graph (AEG, Def. \ref{def:aeg}), which is a deep
embedding of action sets into ILGs of the states. Here, we categorise the action
effects into two types, \emph{unavoidable} and \emph{optional}, and represent
them in the graph. Unavoidable effects are those that are present in all actions
in $\actionSet$, which we can directly apply in the state $s$. Optional effects
are those that are present in some actions in $\actionSet$ but not all, which we
can choose to apply or not. Since our action sets are induced from partial
actions, unavoidable effects are those that only depend on the schema arguments
instantiated so far.

\begin{definition}[AEG]\label{def:aeg}
    Given a state $s$ and a set of actions $\actionSet$. If $\actions_s
        \subseteq \actionSet$, the set $\unavoidableAdd(\actionSet)$ denotes the
    \emph{unavoidable add effects} of $\actionSet$, i.e. $\bigcap_{a \in
            \actionSet} \add(a)$; the set $\unavoidableDel(\actionSet)$ denotes the
    \emph{unavoidable delete effects} of $\actionSet$, i.e. $\bigcap_{a \in
            \actionSet} \del(a)$; the set $\optionalAdd(\actionSet)$ denotes the
    \emph{optional add effects} of $\actionSet$i, i.e. $\bigcup_{a \in
            \actionSet} \add(a) \setminus \unavoidableAdd(\actionSet)$; and lastly
    the set $\optionalDel(\actionSet)$ denotes the \emph{optional delete
        effects} of $\actionSet$, i.e. $\bigcup_{a \in \actionSet} \del(a)
        \setminus \unavoidableDel(\actionSet)$. As a special case, if
    $\actionSet = \actions_s$, then all these sets are empty. This in
    particular happens when the partial action inducing $\actionSet$ is
    $\none$.

    Let $s'$ be the state given by applying the unavoidable effects in $s$, i.e.
    $(s \setminus \unavoidableDel(\actionSet)) \cup
        \unavoidableAdd(\actionSet)$. The \textbf{Action Effect Graph} (AEG) is $\langle V, E, c, l \rangle$ where

    \newcommand{\achieved}{\ensuremath{\mathrm{a}}}
    \newcommand{\unachieved}{\ensuremath{\mathrm{u}}}
    \newcommand{\optionalAddColour}{\ensuremath{\mathrm{oa}}}
    \newcommand{\optionalDelColour}{\ensuremath{\mathrm{od}}}
    \newcommand{\goalColour}{\ensuremath{\mathrm{g}}}
    \newcommand{\nonGoalColour}{\ensuremath{\mathrm{ng}}}

    \begin{itemize}
        \item $V = \objects \cup G \cup s' \cup \optionalAdd(A) \cup
                  \optionalDel(A)$
        \item $E = \bigcup\textstyle_{p = P(o_1, \ldots o_k) \in V \setminus
                      \objects } \{\langle p, o_1 \rangle \ldots, \langle p, o_k
                  \rangle \}$
        \item $c: V \rightarrow 2^\predicates \cup (\{\achieved, \unachieved,
                  \optionalAddColour, \optionalDelColour\} \times \{\goalColour,
                  \nonGoalColour\} \times \predicates)$. For $o \in \objects$,
              $c(o) = \predicates_o$. Otherwise, for $p = P(o_1, \ldots,
                  o_k) \in V \setminus \objects$, $c(p) = (\alpha, \beta, P)$
              where
              \begin{align*}
                  \alpha = \begin{cases}
                               \optionalAddColour, & \mathrm{if\ } p \in \optionalAdd(A)                       \\
                               \optionalDelColour, & \mathrm{if\ } p \in \optionalDel(A)                       \\
                               \unachieved,        & \mathrm{if\ } p \in G \setminus (s' \cup \optionalAdd(A)) \\
                               \achieved,          & \mathrm{otherwise}                                        \\
                           \end{cases}
              \end{align*}
              and $\beta$ is $\goalColour$ if $u \in G$ and $\nonGoalColour$
              otherwise. Here $\achieved$, $\unachieved$, $\optionalAddColour$,
              $\optionalDelColour$ mean \textbf{a}chieved, \textbf{u}nachieved,
              \textbf{o}ptional \textbf{a}dd, and \textbf{o}ptional
              \textbf{d}elete respectively, $\goalColour$ and $\nonGoalColour$
              mean \textbf{g}oal and \textbf{n}on-\textbf{g}oal respectively.
        \item $l: E \mapsto \N$ with $\langle p, o_i \rangle \mapsto i$.
    \end{itemize}
\end{definition}

Unlike the AOAG, the AEG is a deep embedding of the action set into the ILG of
the state -- it considers the effects of the actions in the action set, and
represents them in the graph to reflect how they change the state. An example of
AEG for Blocksworld is shown in Figure \ref{fig:aeg}.

\begin{figure}
    \centering
    \includegraphics[width=0.4742\textwidth]{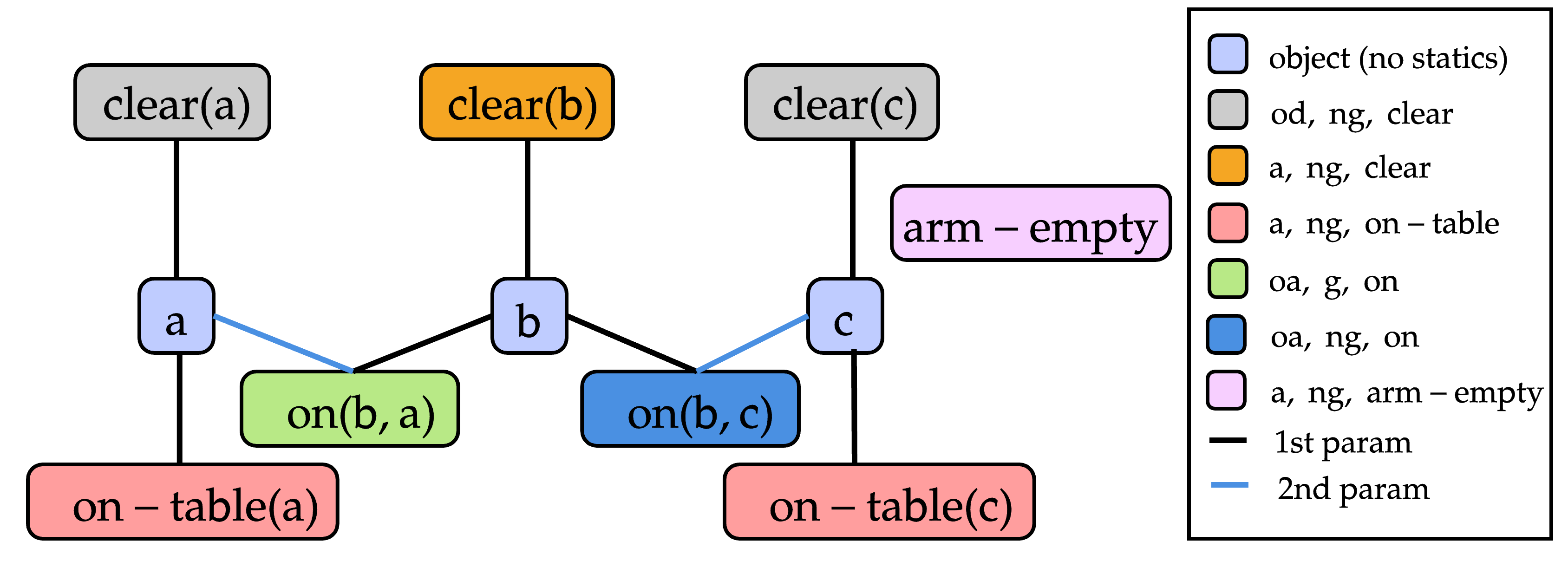}
    \caption{AEG (Def. \ref{def:aeg}) of a Blocksworld Instance with blocks $a$
        and $c$ on the table and $b$ being held, with the goal to place $b$ on
        $a$ and $\actionSet$ containing all applicable $\mathtt{stack}$
        actions.}
    \label{fig:aeg}
\end{figure}

\subsection{Training Action Set Heuristics}\label{sec:training}

We assume that our training data consists of a set classical planning tasks
$\Pi_1, \ldots, \Pi_n$ that share the same planning domain, with corresponding
plans $\pi_1, \ldots, \pi_n$. Each plan is a sequence of actions denoted $\pi_i
    = a_1, \ldots, a_{|\pi_i|}$, which starts at the initial state $s_0$ and visits
the states $s_1, \ldots, s_{|\pi_i|}$. Each such plan represents a sequence of
steps in state-space search. For partial-space search, we decompose each plan
action into a sequence of partial actions of increasing specificity starting at
$\none$, and pair them with the corresponding state that they are applied in.
This gives a sequence of state and partial action sequence pairs
$\langle\!s_0,\!(\rho_{1,\!1}, \ldots,\!\rho_{1, n_1})\rangle,\ldots,\langle
    s_{|\pi_i|-1},\!(\rho_{|\pi_i|,1},\ldots,\rho_{|\pi_i|,n_{|\pi_i|}})\rangle$,
where
$\rho_{j, 1}, \ldots, \rho_{j, n_j}$ are the partial actions obtained from
decomposing $a_j$. Note that $\rho_{j, 1} = \none$ and $\rho_{j, n_j}$ is always
a fully-instantiated action, with $s_{j+1}$ the resulting state of applying it in
$s_j$. Given this, we learn action set heuristics through ranking by extending
the methods proposed in \citet{chen-thiebaux-neurips2024}. We have also explored
doing so through regression in a similar way to what
\citet{chen-thiebaux-neurips2024} did, but found poor performance compared to
ranking. This matches the findings of various works, including
\citet{garrett-et-al-ijcai2016,chrestien-et-al-neurips2023}, and
\citet{hao-et-al-ijcai2024}.

\paragraph{Dataset Generation}

\newcommand{\gap}{\ensuremath{\delta}} \newcommand{\imp}{\ensuremath{\sigma}}

We use $\phi$ to denote the function that maps state and partial action pairs to
feature vectors through either AOAG or AEG using the WL algorithm. Given a state
and partial action sequence pair $\langle s, (\rho_1, \ldots, \rho_n)\rangle$,
we produce a dataset consisting of tuples of the form $\langle \mathbf{x},
    \mathbf{x'}, \gap, \imp \rangle$, where $\mathbf{x}$ and $\mathbf{x'}$ are
feature vectors, $\gap$ is the desired difference between their heuristic
values, and $\imp$ is the importance. Each such tuple is a ranking relation. It
specifies that the heuristic should assign a value to $\mathbf{x}$ that is at
least $\gap$ lower than the value assigned to $\mathbf{x'}$. The value $\imp$
indicates how important it is for the heuristic to satisfy this ranking
relation. Note that the importance weights $\imp$ did not appear in
\citet{chen-thiebaux-neurips2024}, but we found it useful to balance different
types of tuples. There are four types of tuples described below, we distinguish
importance of tuples by the type of the tuple only and not by the specific tuple
itself.

\newcommand{\layerPredecessor}{\ensuremath{\mathrm{lp}}}
\newcommand{\statePredecessor}{\ensuremath{\mathrm{sp}}}
\newcommand{\layerSibling}{\ensuremath{\mathrm{ls}}}
\newcommand{\stateSibling}{\ensuremath{\mathrm{ss}}}

\begin{enumerate}
    \item \emph{Layer predecessors}, which rank later partial actions better than
          earlier partial actions:
          \begin{align*}
                   & \{\langle\phi(s, \rho_1), \phi(s', a'), 1, \imp_\layerPredecessor\rangle\} \\
              \cup &
              \{\langle\phi(s, \rho_{i+1}), \phi(s,
              \rho_{i}), 1, \imp_\layerPredecessor\rangle \mid 1 \le i < n \}
          \end{align*}
          where $s'$ is the state before $s$ in the training plan, and $a'$ is
          the action that transitioned $s'$ to $s$.
    \item \emph{Layer siblings}, which rank partial actions better than those
          with the same specificity but not in the plan-induced sequence. Specifically, for
          each $\rho_i$ in the sequence and for each partial action $\rho'$
          applicable in $s$ but not in the sequence with $\specificity(\rho_i) =
              \specificity(\rho')$, we produce the tuple $\langle\phi(s, \rho_i), \phi(s,
              \rho'), 0, \imp_\layerSibling\rangle$.
    \item \emph{State predecessors}, which rank partial actions in the
          plan-induced sequence better than the state $s$. For each $\rho_i$, we
          produce the tuple $\langle \phi(s, \rho_i), \phi(s, \none), 1,
              \imp_\statePredecessor \rangle$.
    \item \emph{State siblings}, which ranks $\rho_n$ better than other fully
          instantiated partial actions not in the sequence. Specifically, for
          each fully instantiated applicable partial action $\rho' \neq \rho_n$,
          we produce the tuple $\langle \phi(s, \rho_n), \phi(s, \rho'), 0,
              \imp_\stateSibling \rangle$.
\end{enumerate}

\noindent Here, $\imp_\layerPredecessor$, $\imp_\layerSibling$,
$\imp_\statePredecessor$, and $\imp_\stateSibling$ are all hyperparameters.
Layer predecessors establish a chain of ranking relations throughout each
training plan. State predecessors act as fast tracks for such relations, similar
in spirit to skip connections in neural networks. Layer siblings ensure that
partial actions in the sequence are preferred. State siblings emphasising this
for fully instantiated actions. The $\delta$ difference values for the
predecessors tuples are $1$ to indicate that successors are strictly better,
while for siblings tuples it is $0$ to indicate that the siblings may be equally
good.

\newcommand{\asbf}{\alpha}
\newcommand{\objbf}{\beta}

We previously stated that our dataset generation method using partial-space
search yields more data than state-space search. Consider a hypothetical
planning task with $\asbf$ action schemas, each with $k$ parameters, and
$\objbf$ objects that can instantiate any parameter of any schema independently
of other parameters. For a plan of length $n$, the state-of-the-art method by
\citet{hao-et-al-ijcai2024} generates $n \asbf \objbf^k$ tuples. Our method
generates $2n(k+1) + n(\asbf \objbf^k - 1) + n \sum_{i=0}^k (\asbf \objbf^i -
    1)$ tuples, which is approximately $n(k + 2 \asbf \objbf^k)$ tuples.
Furthermore, our method not only generates more data but our data is more
granular, i.e., with a higher ratio of optimal decisions to available decision.
Specifically, \citet{hao-et-al-ijcai2024}'s dataset encodes $n$ optimal actions
resulting in a ratio of $1/(\asbf \objbf^k)$, i.e., the correct action choice is
always 1 out of the branching factor of the grounded actions $\asbf \objbf^k =
    |\actions|$. In contrast, our method encodes $n(k+1)$ optimal decisions (i.e.,
optimal action schema, optimal first parameter, etc.) resulting in an
approximate density of $(k+1)/(k+2 \asbf \objbf^k)$. For $\asbf \objbf^k (k-1) >
    k$, our dataset is denser, i.e., offers more optimal decisions to guide the
learning process. Note that this condition almost always holds since $\asbf
    \objbf^k = |\actions| \gg 2$.

\paragraph{Ranking Model}

Our complete dataset is the collection of all tuples from all state and partial
action sequence pairs, namely $\Dc = \{(\mathbf{x}_i, \mathbf{x'}_i, \gap_i,
    \imp_i) \mid i \in [|\Dc|]\}$. Once having generated this, we learn a linear
model with weights $\mathbf{w}$ by solving the following linear program,
\begin{gather*}
    \min_{\mathbf{w} \in \R^n, \mathbf{z} \in \R^{|\Dc|}_{\ge 0}} \quad C \sum_i \imp_i \mathbf{z}_i + \lVert \mathbf{w} \rVert_1 \\
    \begin{align*}
        \textup{s.t.} \quad \mathbf{w}^T (\mathbf{x}_i - \mathbf{x'}_i) & \ge \gap_i - \mathbf{z}_i & \forall i
    \end{align*}
\end{gather*}
Here, $C$ is a regularisation hyperparameter, and the variables $\mathbf{z}$
represent the slack variables for each of the constraints given by the ranking
relations. The term $\lVert \mathbf{w} \rVert_1$ is not linear itself, but it is
trivially encoded using auxiliary variables and linear constraints. The
optimisation target represents that, for all $i$, $\mathbf{w}$ should map
$\mathbf{x}_i$ and $\mathbf{x'}_i$ to values that are at least $\gap_i$ apart as
much as possible according to the importance $\imp_i$ (first term), while being
subject to regularisation (second term).

In practice, this linear program can be very large. We use constraint and column
generations techniques adapted from \citet{dedieu-et-al-jmlr2022} to solve it
efficiently. Moreover, we tune the hyperparameter $C$ automatically. Our
training instances are ordered roughly by problem size in ascending order. We
use the first 80\% of instances as our training set and the rest as our
validation set. We tune $C$ by applying Bayesian optimisation to minimise the
objective value $\sum_i \imp_i \mathbf{z}_i$ on the validation set.

Once $\mathbf{w}$ is learned, we obtain an action set heuristic $h_{\mathbf{w}}$
that first maps state and action set pairs to feature vectors, then uses the dot
product of the feature vectors and $\mathbf{w}$ as the heuristic value.

\section{Experimental Evaluation}

We implemented partial-space search and action set heuristics in a new planner,
LazyLifted, and evaluated it against the state-of-the-art. We test its
performance using benchmarks from the International Planning Competition (IPC)
2023 learning track \citep{taitler-et-al-aimag2024} and new high branching
factor benchmarks. The IPC set contains 10 domains, including Blocksworld,
Ferry, and Transport, among others.

The high branching factor (HBF) set contains 5 domains. It includes a version of
Blocksworld, Blocksworld Large, with many blocks that are irrelevant to the
goal, which induces a high branching factor. This is adapted from the
Blocksworld Large benchmarks by \citet{lauer-et-al-ijcai2021} to fit the IPC
2023 learning track style. We also include three versions of Transport,
Transport Sparse, Transport Dense, and Transport Full, which reflect the
underlying graphs are sparse, dense, and fully connected. Branching factor on
Transport increases with graph density. Lastly, we introduce a new domain,
Warehouse, which models multiple stacks of boxes with some needing to be
removed. Warehouse resembles Blocksworld in the modelling of box stacks, but
differs as boxes can move in a single action between the tops of two stacks.
This results in a high branching factor, quadratic to the number of stacks,
which is high even for small instances.

Each IPC or HBF benchmark contains up to 99 training instances and 90 test
instances. The test instances are ordered and grow in size rapidly. The training
instances are roughly the same size as the smallest 30 test instances. A
detailed description of the benchmarks, including their problem sizes and
branching factors, is in the appendix. Our code and benchmarks will be made
available upon publication.

Our learning-based heuristics require sample plans on training instances. For
the IPC set, we use the same training plans as \citet{chen-et-al-icaps2024},
which are generated by the Scorpion optimal planner \citep{seipp-et-al-jair2020}
under a 30 minutes timeout for each instance. For the HBF set, we generate
training plans for the Transport variants using the same method. For Blocksworld
Large and Warehouse, Scorpion cannot solve most of the training instances in the
time limit, so we use the first plan returned by LAMA
\citep{richter-westphal-jair2010} under the same time limit.

LazyLifted has importance hyperparameters for dataset generation
(Sec.~\ref{sec:training}). We set these
to $\imp_\layerPredecessor = 0.5$, $\imp_\layerSibling = 2.0$,
$\imp_\statePredecessor = 0.5$, and $\imp_\stateSibling = 1.0$ when training
AOAG based heuristics and to $\imp_\layerPredecessor = 0.5$, $\imp_\layerSibling
    = 2.0$, $\imp_\statePredecessor = 0.5$, and $\imp_\stateSibling = 0.75$ when
training AEG based heuristics. These values were found to work well across all
domains except Warehouse. On Warehouse, the high branching factor even on small
training instances meant that there are a lot of siblings, which sometimes made
the optimal weights close to zero. We use $\imp_\layerPredecessor = 2.0$,
$\imp_\layerSibling =1.5$, $\imp_\statePredecessor = 2.0$, and
$\imp_\stateSibling = 0.5$ for Warehouse for both AOAG and AEG. Lastly, we use
two iterations of the WL algorithm for training the heuristic.

We compare LazyLifted against state-of-the-art baselines LAMA, Powerlifted
(PWL), and WL-GOOSE. LAMA is a strong satisficing planner that uses multiple
heuristics with additional optimisation techniques
\citep{richter-westphal-jair2010}. We use LAMA with its first plan output.
Powerlifted is a lifted planner that specialises in solving hard-to-ground
planning tasks, which is particularly relevant as many high branching factor
tasks are also hard-to-ground, and hence require lifted planning. WL-GOOSE
learns state-of-the-art heuristics for state-space search to guide GBFS
\citep{chen-et-al-icaps2024}. We use our own implementation of WL-GOOSE, which
performs slightly better than the original version. We trialed several
configurations for WL-GOOSE, namely the configuration used in original paper,
the recommended configuration the WL-GOOSE codebase\footnote{Available from
    \url{https://github.com/DillonZChen/goose}.}, and a variation of the recommended
configuration that only uses 2 WL iterations instead of 3 to match LazyLifted.
Ultimately, we only present results for the last configuration as it has the
best overall performance.

We do not consider other learning-based planners as all that we are aware of
perform noticeably worse than WL-GOOSE on the IPC set as shown either by
\citet{chen-et-al-icaps2024} or the results in the IPC 2023 learning track
\citep{taitler-et-al-aimag2024}.

Moreover, since action set heuristics can be used as
state-space heuristics, we run all our action set heuristics on both state and
partial-space search (denoted \stateSpaceSearch and \partialSpaceSearch) to
isolate the effect of the search space. This means that we compare both the
state-space $\hFF$ heuristic (denoted $\stateSpaceSearch$-FF) and the
partial-space $\restrictedFF$ heuristic (denoted $\partialSpaceSearch$-FF). It
also means that we train two action set heuristics, one for AOAG and one for
AEG, and use them on both search spaces, resulting in four possible
configurations.

All experiments are conducted on a cluster with Intel Xeon 3.2 GHz CPU cores.
Training for both LazyLifted and WL-GOOSE is performed using 1 core with 32 GB
of memory with a 12-hour timeout, excluding data generation, i.e., solving
training instances. Both data generation and solving testing problems are done
using 1 core and 8 GB of memory with a 30-minutes timeout per problem.

We evaluate planners on two metrics, coverage and quality. Coverage is the
number of problem solved and the results per domain are shown in Table
\ref{tab:overview-coverage}. The quality score of a planner on a task is $0$
if it does not find a plan and $C^* / C$ otherwise, where $C$ is the cost of the
plan found by the planner and $C^*$ is the cost of the best plan found by any
planner in the experiment. Table \ref{tab:overview-quality} shows the sum of
quality scores per domain. It is worthnoting that, due to the benchmark
setup where the instances in each domain scale up in size quickly, even a small
increase either metric can represent a significant improvement. See the Appendix for additional results.

\begin{table}[tbp]
    \centering
    \small
    
    \setlength{\tabcolsep}{2pt}
    \begin{tabular}{clccccccccc}
        \toprule
         &  & \multicolumn{4}{c}{Baseline} & \multicolumn{5}{c}{New} \\
        \cmidrule(lr){7-11} \cmidrule(lr){3-6} Set & Domain & \rotatebox[origin=c]{90}{LAMA} & \rotatebox[origin=c]{90}{WL-GOOSE} & \rotatebox[origin=c]{90}{PWL} & \rotatebox[origin=c]{90}{$\mathrm{S^3}$-FF} & \rotatebox[origin=c]{90}{$\mathrm{S^3}$-AOAG} & \rotatebox[origin=c]{90}{$\mathrm{S^3}$-AEG} & \rotatebox[origin=c]{90}{$\mathrm{PS^2}$-FF} & \rotatebox[origin=c]{90}{$\mathrm{PS^2}$-AOAG} & \rotatebox[origin=c]{90}{$\mathrm{PS^2}$-AEG} \\
        \midrule
        \multirow{10}{*}{\rotatebox[origin=c]{90}{IPC23 LT}} & blocksworld & 60 & 77 & 37 & 29 & 89\cellcolor{second} & \textbf{90}\cellcolor{first} & 24 & 89\cellcolor{second} & 84\cellcolor{third} \\
         & childsnack & 35 & 36\cellcolor{third} & 0 & 16 & 37\cellcolor{second} & \textbf{38}\cellcolor{first} & 14 & 9 & 15 \\
         & ferry & 68 & \textbf{90}\cellcolor{first} & 0 & 60 & \textbf{90}\cellcolor{first} & 88\cellcolor{second} & 49 & \textbf{90}\cellcolor{first} & 87\cellcolor{third} \\
         & floortile & \textbf{11}\cellcolor{first} & 3 & 9\cellcolor{third} & 10\cellcolor{second} & 1 & 1 & 6 & 1 & 1 \\
         & miconic & \textbf{90}\cellcolor{first} & \textbf{90}\cellcolor{first} & 81\cellcolor{second} & 77\cellcolor{third} & \textbf{90}\cellcolor{first} & \textbf{90}\cellcolor{first} & 68 & \textbf{90}\cellcolor{first} & \textbf{90}\cellcolor{first} \\
         & rovers & \textbf{69}\cellcolor{first} & 34\cellcolor{third} & 53\cellcolor{second} & 28 & 34\cellcolor{third} & 32 & 34\cellcolor{third} & 34\cellcolor{third} & 26 \\
         & satellite & \textbf{89}\cellcolor{first} & 41 & 0 & 49 & 39 & 53\cellcolor{second} & 50\cellcolor{third} & 29 & 41 \\
         & sokoban & \textbf{40}\cellcolor{first} & 27 & 34\cellcolor{second} & 32\cellcolor{third} & 27 & 27 & 29 & 27 & 24 \\
         & spanner & 30\cellcolor{third} & 71\cellcolor{second} & 30\cellcolor{third} & 30\cellcolor{third} & 71\cellcolor{second} & 71\cellcolor{second} & 30\cellcolor{third} & \textbf{72}\cellcolor{first} & \textbf{72}\cellcolor{first} \\
         & transport & \textbf{66}\cellcolor{first} & 47 & 53\cellcolor{third} & 36 & 53\cellcolor{third} & 50 & 38 & 54\cellcolor{second} & 54\cellcolor{second} \\
        \midrule
        \multicolumn{2}{c}{\textbf{sum IPC coverage}} & \textbf{558}\cellcolor{first} & 516 & 297 & 367 & 531\cellcolor{third} & 540\cellcolor{second} & 342 & 495 & 494 \\
        \midrule
        \multirow{5}{*}{\rotatebox[origin=c]{90}{HBF}} & blocksworld-large & 7 & 31 & 2 & 0 & 35\cellcolor{third} & 18 & 0 & \textbf{74}\cellcolor{first} & 48\cellcolor{second} \\
         & transport-sparse & 62\cellcolor{second} & 37 & 58\cellcolor{third} & 31 & 41 & 37 & 31 & \textbf{64}\cellcolor{first} & 58\cellcolor{third} \\
         & transport-dense & \textbf{66}\cellcolor{first} & 57\cellcolor{third} & 52 & 39 & 59\cellcolor{second} & 51 & 36 & 57\cellcolor{third} & 57\cellcolor{third} \\
         & transport-full & \textbf{66}\cellcolor{first} & 58 & 0 & 42 & 63\cellcolor{second} & 61\cellcolor{third} & 41 & 55 & 60 \\
         & warehouse & 30 & 15 & \textbf{90}\cellcolor{first} & 35 & 58\cellcolor{third} & 27 & 54 & 79\cellcolor{second} & 49 \\
        \midrule
        \multicolumn{2}{c}{\textbf{sum HBF coverage}} & 231 & 198 & 202 & 147 & 256\cellcolor{third} & 194 & 162 & \textbf{329}\cellcolor{first} & 272\cellcolor{second} \\
        \midrule
        \multicolumn{2}{c}{\textbf{sum total coverage}} & 789\cellcolor{second} & 714 & 499 & 514 & 787\cellcolor{third} & 734 & 504 & \textbf{824}\cellcolor{first} & 766 \\
        \bottomrule
    \end{tabular}
    \caption{Coverage of various planning systems. The best score for each row is highlighted in bold. The top three unique scores for each row are highlighted in different shades of green with darker being better. }\label{tab:overview-coverage}
\end{table}

\begin{table}[tbp]
    \centering
    \small
    
    \setlength{\tabcolsep}{2pt}
    \begin{tabular}{clccccccccc}
        \toprule
         &  & \multicolumn{4}{c}{Baseline} & \multicolumn{5}{c}{New} \\
        \cmidrule(lr){7-11} \cmidrule(lr){3-6} Set & Domain & \rotatebox[origin=c]{90}{LAMA} & \rotatebox[origin=c]{90}{WL-GOOSE} & \rotatebox[origin=c]{90}{PWL} & \rotatebox[origin=c]{90}{$\mathrm{S^3}$-FF} & \rotatebox[origin=c]{90}{$\mathrm{S^3}$-AOAG} & \rotatebox[origin=c]{90}{$\mathrm{S^3}$-AEG} & \rotatebox[origin=c]{90}{$\mathrm{PS^2}$-FF} & \rotatebox[origin=c]{90}{$\mathrm{PS^2}$-AOAG} & \rotatebox[origin=c]{90}{$\mathrm{PS^2}$-AEG} \\
        \midrule
        \multirow{10}{*}{\rotatebox[origin=c]{90}{IPC23 LT}} & blocksworld & 36 & 73 & 21 & 13 & 88\cellcolor{third} & \textbf{89}\cellcolor{first} & 12 & 88\cellcolor{second} & 82 \\
         & childsnack & 22 & 36\cellcolor{third} & 0 & 12 & 37\cellcolor{second} & \textbf{38}\cellcolor{first} & 8 & 9 & 14 \\
         & ferry & 57 & 87\cellcolor{third} & 0 & 55 & 89\cellcolor{second} & 86 & 33 & \textbf{90}\cellcolor{first} & 86 \\
         & floortile & \textbf{10}\cellcolor{first} & 3 & 9\cellcolor{third} & 9\cellcolor{second} & 1 & 1 & 4 & 0 & 1 \\
         & miconic & 73 & \textbf{89}\cellcolor{first} & 77 & 76 & \textbf{89}\cellcolor{first} & \textbf{89}\cellcolor{first} & 56 & 82\cellcolor{second} & 82\cellcolor{third} \\
         & rovers & \textbf{67}\cellcolor{first} & 21 & 52\cellcolor{second} & 25 & 19 & 17 & 30\cellcolor{third} & 17 & 13 \\
         & satellite & \textbf{85}\cellcolor{first} & 31 & 0 & 47\cellcolor{second} & 26 & 41 & 43\cellcolor{third} & 17 & 25 \\
         & sokoban & \textbf{35}\cellcolor{first} & 19 & 29\cellcolor{second} & 27\cellcolor{third} & 19 & 19 & 23 & 14 & 12 \\
         & spanner & 30 & 67\cellcolor{third} & 30 & 30 & 67\cellcolor{third} & 67\cellcolor{third} & 30 & 68\cellcolor{second} & \textbf{68}\cellcolor{first} \\
         & transport & \textbf{63}\cellcolor{first} & 39 & 47\cellcolor{second} & 29 & 43 & 43 & 26 & 41 & 44\cellcolor{third} \\
        \midrule
        \multicolumn{2}{c}{\textbf{sum IPC quality}} & 480\cellcolor{second} & 465 & 264 & 324 & 478\cellcolor{third} & \textbf{490}\cellcolor{first} & 265 & 426 & 426 \\
        \midrule
        \multirow{5}{*}{\rotatebox[origin=c]{90}{HBF}} & blocksworld-large & 7 & 22 & 2 & 0 & 27\cellcolor{third} & 12 & 0 & \textbf{64}\cellcolor{first} & 39\cellcolor{second} \\
         & transport-sparse & \textbf{59}\cellcolor{first} & 31 & 48\cellcolor{second} & 21 & 27 & 23 & 21 & 47\cellcolor{third} & 38 \\
         & transport-dense & \textbf{63}\cellcolor{first} & 48\cellcolor{third} & 46 & 31 & 50\cellcolor{second} & 47 & 25 & 47 & 36 \\
         & transport-full & \textbf{63}\cellcolor{first} & 48 & 0 & 39 & 52\cellcolor{third} & 57\cellcolor{second} & 32 & 42 & 40 \\
         & warehouse & 30 & 6 & \textbf{89}\cellcolor{first} & 35 & 54 & 14 & 54\cellcolor{third} & 57\cellcolor{second} & 17 \\
        \midrule
        \multicolumn{2}{c}{\textbf{sum HBF quality}} & 222\cellcolor{second} & 154 & 185 & 126 & 209\cellcolor{third} & 153 & 132 & \textbf{257}\cellcolor{first} & 171 \\
        \midrule
        \multicolumn{2}{c}{\textbf{sum total quality}} & \textbf{702}\cellcolor{first} & 619 & 449 & 449 & 687\cellcolor{second} & 643 & 398 & 683\cellcolor{third} & 597 \\
        \bottomrule
    \end{tabular}
    \caption{Quality score of various planning systems rounded to integers. The best score for each row is highlighted in bold. The top three unique scores for each row are highlighted in different shades of green with darker being better. }\label{tab:overview-quality}
\end{table}

\paragraph{Comparison against baselines}

Overall, our learned action set heuristics outperform WL-GOOSE in both total
coverage and total quality scores across both partial and state-space search. On
the IPC set, only learned action set heuristics with state-space search
($\stateSpaceSearch$-AOAG/AEG) outperform WL-GOOSE by around 5\% in coverage and
quality. The improvements are more substantial on the HBF set, where our learned
action set heuristics with partial-space search (specifically,
$\partialSpaceSearch$-AOAG) are up to 60\% better in coverage and quality scores
than WL-GOOSE. In the combined set of problems, \stateSpaceSearch-AOAG and
\partialSpaceSearch-AOAG obtained 10\% to 15\% better total coverage and 10\%
improvement in total quality over WL-GOOSE. Additional results available in
Appendix show that the learned action set heuristics perform similarly to
WL-GOOSE on runtime overall, with the comparison varying widely for specific
domains. Altogether, given WL-GOOSE's state-of-the-art status in learned
heuristics, these results represent a significant improvement.

LazyLifted also outperforms LAMA in several cases, namely coverage and quality
for the HBF set, quality for the IPC set, and total coverage when considering
the combined set. LazyLifted is notably strong on IPC domains such as Blocksworld,
Ferry, and Spanner, and the non-Transport HBF domains, while still lacking
behind on domains such as Rovers and Satellite. It is noteworthy that LazyLifted
finds low-quality plans on Transport domains compared to coverage, indicating
that LazyLifted's learned action set heuristics are not able to outperform LAMA
on these domains. In most other cases, LazyLifted finds plans with reasonable
quality. However, the large makeup of Transport domains in the benchmark sets
means that LazyLifted outperforms LAMA only in total coverage but not total
quality.

It is important when interpreting these results to note that LAMA employs
various planning techniques, e.g., multi-queue search and preferred operators,
while our system only uses a single heuristic with GBFS over either partial or
state-space search. In theory, techniques used in LAMA could also be applied in
LazyLifted to further enhance performance. We do not use them to isolate the
effect of our contributions over existing techniques.

Lastly, LazyLifted performs significantly better than Powerlifted (PWL) on both
the IPC set and HBF set. The improvement over Powerlifted on the IPC set is
unsurprising given that Powerlifted is not the state-of-the-art in general
classical planning. However, given the state-of-the-art nature of Powerlifted in
hard-to-ground planning, the improvement over Powerlifted on the HBF set
indicates improvements LazyLifted makes in this space.

\paragraph{Partial-space search versus state-space search}

Our experiments are set up to isolate the effect of partial-space search over
state-space search by running the same heuristics on both search spaces. Our results show
that partial-space search is not beneficial on the IPC set, as indicated by
coverage and quality scores. On the HBF set, partial-space search is
unsurprisingly very impactful, as indicated by the substantial improvement in
coverage and quality scores.

To further investigate, we compared the branching factor (average number of
successors per expansion), number of expansions, and number of evaluations
between partial and state-space search (see appendix for full results). In
terms of median, on the combined IPC and HBF set, partial-space search reduces
the branching factor by 88\%, increases the number of expansions by around 7
times, and reduces the number of evaluations by approximately 14\%. This
indicates that partial-space search can reduce the overall work required to
solve planning tasks, which is typically dominated by the number of
evaluations. These results also demonstrated that the IPC domains tend to have
a lower branching factor, leading to more evaluations for partial-space search.
This explains why partial-space search is not beneficial on the IPC set.

\paragraph{Informedness of learned action set heuristics}

Our experiments also highlight the effectiveness of partial-space
search at training heuristics. The only difference between WL-GOOSE and
$\stateSpaceSearch$-AOAG/AEG is the heuristic used, where the latter uses
partial-space search to train action set heuristics, then uses these action
set heuristics as regular state-space heuristics. The learned action set
heuristics with state-space search generally outperform WL-GOOSE on both
coverage and quality. Moreover, they also generally have less search node
expansions and heuristic evaluations (see Appendix). These results overall
indicate that training with partial-space search leads to more informed
heuristics. We hypothesise that this improvement partially stems from the
larger training datasets generated by LazyLifted, which are on average 2.62
times larger than those generated by WL-GOOSE from the same training
plans.

\paragraph{AOAG versus AEG}

Our results show that AEG works slightly better on the IPC set, while AOAG works
much better on the HBF set, regardless of the search space. We hypothesise that
this is due to the average of size of action sets on these domains and the
difference in the way the two graphs encode action sets. AEG encodes action sets
more deeply into the graph by considering action effects, but does not label in
the graph which optional effect belongs to which actions. This likely works
better under low branching factors, as there is less ambiguity between effects
of different actions. In contrast, the shallow action encoding of AOAG is
more explicit, which likely works better under high branching factors by
avoiding this ambiguity.

\paragraph{How well does $\hFF$ translate to partial-space search?}

Another case of interest is the comparison between $\restrictedFF$
(\partialSpaceSearch-FF) and $\hFF$ (\stateSpaceSearch-FF): $\restrictedFF$ is
slightly worse than $\hFF$ on the IPC set and slightly better than HBF set. This
is consistent with the effects of partial-space search, showing that
$\restrictedFF$ is a faithful translation of $\hFF$ to partial-space search.

\paragraph{On the HBF Transport variants, why does coverage of partial-space
    search decrease as graph density increases?}

We observed that as graph density increases, the branching factor increases,
which should favour partial-space search. However, the path-finding problem also
becomes much easier, meaning that the high branching factor costs on dense
graphs is paid far less often than on sparse graphs, making partial-space search
overall less effective.

\paragraph{Why do learned action set heuristics not perform well on domains such as Rovers and Satellite?}

Our methods for learning action set heuristics are based on those used in
WL-GOOSE \citep{chen-et-al-icaps2024}, namely feature extraction using the WL
algorithm, and we inherit weaknesses of their methods. Our learned action set
heuristics perform similarly to WL-GOOSE on these domains, indicating that our
methods are likely not the cause of the poor performance.

\section{Conclusion, Related, and Future Work}

We introduced partial-space search for planning, which is a new search space
that allows search and learned heuristics to be more deeply integrated.
Partial-space search decomposes branching factor, allows more training data to be
generated out of the same training plans, and is able to shift the evaluation
speed versus informedness trade-off of heuristics to favour informedness. This
results in significant empirical improvements, particularly under high
branching factors, over existing learned heuristics and even outperforms the
state-of-the-art non-learning planner LAMA, which indicates the potential of
deep integrations between planning and learning.

Partial-space search requires specialised action set heuristics that evaluate a
set of actions on a state. We showed how partial-space search can be generally
useful by providing a method to automatically convert existing state-space
heuristics to action set heuristics. Moreover, we introduced novel graph
representations that encapsulate action set semantics in various ways for
learning well-informed action set heuristics. These learned action set
heuristics are effective whether employed in partial or state-space search.

In terms of related work, learning for planning, particularly in the form of
learned heuristics, has gathered notable research attention. Previous works have
developed methods for learning state-space heuristics
\citep{chen-et-al-icaps2024} and explored possible alternative heuristic targets
\citep{garrett-et-al-ijcai2016,ferber-et-al-icaps2022,chrestien-et-al-neurips2023,hao-et-al-ijcai2024}.
Yet to our knowledge, this is the first work focused on how to design a search
algorithm based on the capabilities of learned heuristics.

For dealing with large branching factors, \citet{yoshizumi-et-al-aaai2000}
proposed a variant of the A* search algorithm (PEA*) that partially expands
search nodes to reduce the memory cost induced by high branching factors.
\citet{goldenberg-et-al-jair2014} extended PEA* to also reduce the time cost
of search, but required domain and heuristic knowledge. Their works had
focused on modifying the search algorithm, while our partial-space search
modifies the search space by incorporating heuristic predictions
over sets of actions and are hence complementary.

In the context of Monte-Carlo Tree Search, \citet{shen-et-al-socs2019} explored
how to guide search with generalised policies. Their work is similar to our as
they also explore alternative forms of heuristics, in their case, a generalised
policy that predicates a probability distribution over applicable actions in a
state. Similar to how partial-space search introduces a different interface
between search and heuristics over the traditional state-based interface, their
approach uses a probability distribution as the interface. Their work is
focused on the application of search to overcome weaknesses in learned
generalised policies, while our work redesigns search itself to better suit
learned heuristics. Their work also achieved limited empirical success, while
we demonstrated significant empirical improvements over existing learning-based
methods.

In future work, we hope to extend the applicability of partial-space search and
explore other forms of deep integrations between planning and learning.
Extending the applicability involves finding efficient adaptations of existing
heuristics and other search techniques to partial-space search. It also involves
extending partial-space search to other planning paradigms, such as numeric
planning and planning under uncertainty. Exploring other forms of deep integrations
involves exploring other ways learning and planning can help each other, such as
using learning to guide planners to focus on specific subtasks of planning
tasks.

\section*{Acknowledgments}

The computing resources for this work was supported by the Australian Government
through the National Computational Infrastructure (NCI) under the ANU Startup
Scheme.

\bibliography{paper}

\appendix
\section{Detailed Explanation of $\restrictedFF$}\label{app:restricted-ff}

\newcommand{\Cc}{\ensuremath{\mathcal{C}}}
\newcommand{\Rc}{\ensuremath{\mathcal{R}}}
\newcommand{\Fc}{\ensuremath{\mathcal{F}}}

To explain the restricted FF heuristic $\restrictedFF$, we first introduce some
background on the lifted computation of the FF heuristic $\hFF$.

\subsection{Background}

A Datalog \emph{rule} $r$ has the form $P \leftarrow Q_1, \ldots, Q_n$, where
$P$ is the head of the rule $head(r)$ and $Q_1, \ldots, Q_n$ are the body of the
rule $body(r)$. Each term $P, Q_1, \ldots, Q_n$ is a Datalog \emph{predicate}
with some argument variables. Given a set of \emph{constants} $\Cc$, the rule
$r$ can be grounded by replacing each argument variable with a constant from
$\Cc$. Each subrule has the semantics that the head predicate is true if all
body predicates are true. A Datalog \emph{program} $\Dc = \langle \Fc, \Rc
    \rangle$ consists of a set of \emph{facts} $\Fc$ and a set of rules $\Rc$. The
facts in $\Fc$ represent ground predicates that are initially true, and rules in
$\Rc$ define how to derive new ground predicates from existing ones.

\citet{correa-et-al-aaai2022} showed that the well-known FF heuristic $\hFF$
can be computed in a lifted fashion using Datalog. For each planning task $\Pi$,
they define a Datalog program $\Dc_\Pi$, whose facts $\Fc_\Pi$ represent the
current state of the planning task, and rules $\Rc_\Pi$ correspond to the
action schemas of the planning task. Their computation relies on first
preprocessing the planning task to generate the Datalog program $\Dc_\Pi$. For
each heuristic evaluation, they change the facts in $\Fc_\Pi$ to represent the
state being evaluated, and then run a modified Depth First Search (DFS) to
compute the FF heuristic value.

\subsection{Restricted FF Heuristic}

Our restricted FF heuristic $\restrictedFF$ is a variant of the FF heuristic
where we would like to restrict the initial set of actions that can be applied
in the state being evaluated to a subset $\actionSet$, as defined in Def. 4 of
the main paper. To do so, we first modify the planning task according to the
first two modifications in Def. 4. We use this modified planning task to
generate the preprocessed Datalog program in the same way as
\citet{correa-et-al-aaai2022}. For each heuristic evaluation, we add temporary
Datalog rules representing the action schemas that would be added in
modification 3 of Def. 4. These temporary rules are added in the same way as
\citet{correa-et-al-aaai2022}. We then run the modified DFS to compute the FF
heuristic value. The restricted FF heuristic $\restrictedFF$ is the value
computed by the modified DFS. Once computed, we remove the temporary rules and
return the heuristic value.

\section{Domain Description}\label{app:domain-description}

A detailed description of the IPC 2023 learning track domains is available from
\citet{taitler-et-al-aimag2024}. Below we describe the high branching factor
domains.

\paragraph{Blocksworld Large}

The Blocksworld Large domain is a variant of the classical Blocksworld domain
with a larger number of blocks. The domain consists of a number of blocks
arranged into towers, with the goal to change the arrangement of the blocks to
match a goal configuration. Blocksworld Large differs from the Blocksworld
benchmarks in the IPC set in that there are a lot more blocks, yet the problem
is generally not much harder since most blocks are not in the goal.

\paragraph{Transport variants}

The original Transport domain in the IPC set involves a number of trucks that
must move packages between locations in a graph, subject to a capacity limit for
each truck. The Transport variants in the HBF set are identical to the original
tasks, except that the density of the graphs is changed. In the original IPC
set, the number of edges is randomly selected from $V-1$ to $V(V-1)/2$, where
$V$ is the number of vertices. In the Sparse variant, the number of edges is
is randomly selected from $V-1$ to $1.5(V-1)$. In the Dense variant, the number
of edges is randomly selected from $0.5V(V-1) / 2$ to $0.8V(V-1) / 2$. The Full
variant has fully connected graphs. In all variants and the IPC set, the graphs
do not contain disconnected components.

\paragraph{Warehouse}

Warehouse is a new domain that we introduced in the HBF set. The domain involves
a number of boxes stacked together, similar to Blocksworld. However, the goal is
to remove a subset of boxes, i.e., moving them out of the warehouse. Only the
boxes marked for removal can be removed. Moreover, unlike Blocksworld, boxes on
top of stacks can be moved to another stack in a single action, which makes the
branching factor quadratic to the number of stacks rather than linear, as in
Blocksworld. This creates an extremely high branching factor, even for small
tasks.

Table~\ref{tab:concise-domain-description} shows the average branching factor and a proxy for the range of the problem's size for each domain in our experiments.

\begin{table}[!ht]
    \centering
    \footnotesize

    \setlength{\tabcolsep}{0.7mm}
    \begin{tabular}{lcccc}
        \toprule
                                                     & \multicolumn{2}{c}{Test (easy)} & \multicolumn{2}{c}{Test (medium/hard)}                                   \\
        \cmidrule(lr){4-5} \cmidrule(lr){2-3} Domain & Size                            & Bran. Fac.                             & Size        & Bran. Fac.        \\
        \midrule
        blocksworld                                  & [5, 30]                         & [1.7, 6.3]                             & [35, 500]   & [4.6, 61.0]       \\
        childsnack                                   & [4, 10]                         & [1.8, 18.1]                            & [15, 300]   & [1680, 26287] \\
        ferry                                        & [1, 20]                         & [1.8, 5.3]                             & [10, 1000]  & [5.4, 223]      \\
        floortile                                    & [2, 30]                         & [1.9, 3.5]                             & [100, 952]  & [5.6, 212]      \\
        miconic                                      & [1, 10]                         & [1.6, 11.4]                            & [20, 500]   & [17.4, 146]     \\
        rovers                                       & [1, 4]                          & [3.8, 30.2]                            & [5, 30]     & [19.1, 726]     \\
        satellite                                    & [3, 10]                         & [10.0, 103]                          & [15, 100]   & [199, 3532]   \\
        sokoban                                      & [1, 4]                          & [1.2, 2.8]                             & [4, 80]     & [1.2, 2.0]        \\
        spanner                                      & [1, 5]                          & [1.1, 1.9]                             & [15, 250]   & [1.1, 44.4]       \\
        transport                                    & [3, 6]                          & [5.0, 79.6]                            & [10, 50]    & [27.3, 1669]    \\
        bw-large                            & [500, 703]                      & [102, 159]                         & [720, 1221] & [113, 161]    \\
        trpt-sparse                             & [3, 6]                          & [3.7, 13.8]                            & [10, 50]    & [10.1, 195]     \\
        trpt-dense                              & [3, 6]                          & [7.0, 39.8]                            & [10, 50]    & [105, 1653]   \\
        trpt-full                               & [3, 6]                          & [14.8, 53.2]                           & [10, 50]    & [180, 1879]   \\
        warehouse                                    & [4, 40]                         & [7.3, 1178]                          & [40, 80]    & [1413, 2079]  \\
        \bottomrule
    \end{tabular}
    \caption{Benchmark domain size and branching factor descriptions. Size is measured by number of key objects, e.g, blocks in Blocksworld. Branching factor is the average number of successors generated per state when using state space search with $h^{\mathrm{FF}}$.}\label{tab:concise-domain-description}
\end{table}

\section{More Experimental Results}\label{app:comparisons}

Figure~\ref{fig:comparisons} shows the comparison between partial-space search and state-space search in terms of branching factor, expansions and evaluations, for AOAG and AEG.
Figure~\ref{fig:heuristic-comparisons} presents the comparison of learned action set heuristics with state-space search versus WL-GOOSE on evaluations and expansions (bottom).
Lastly, Figure~\ref{fig:comparisons} shows the comparison of learned action set heuristics versus WL-GOOSE on runtime.

\begin{figure*}[!ht]
    \centering
    \includegraphics[width=\textwidth]{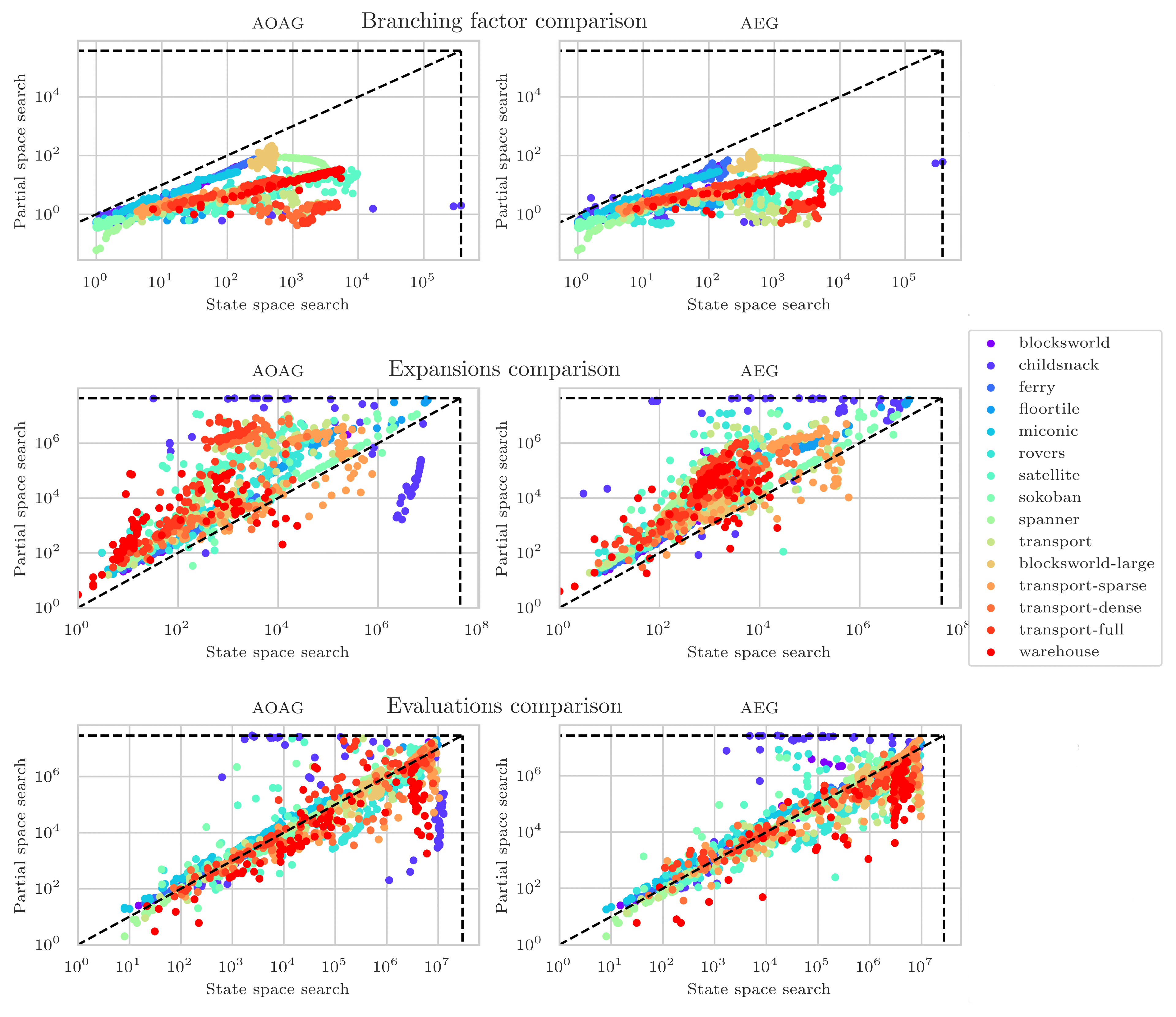}
    \caption{Comparison of partial-space search versus state-space search on
        branching factor (top, measured as average number of successors per
        expansion), expansions (middle), and evaluations (bottom), for AOAG and
        AEG. The $x=y$ line is shown in the diagonals. Points below it favour
        partial-space search and above it favour state-space search.}
    \label{fig:comparisons}
\end{figure*}

\begin{figure*}[!ht]
    \centering
    \includegraphics[width=\textwidth]{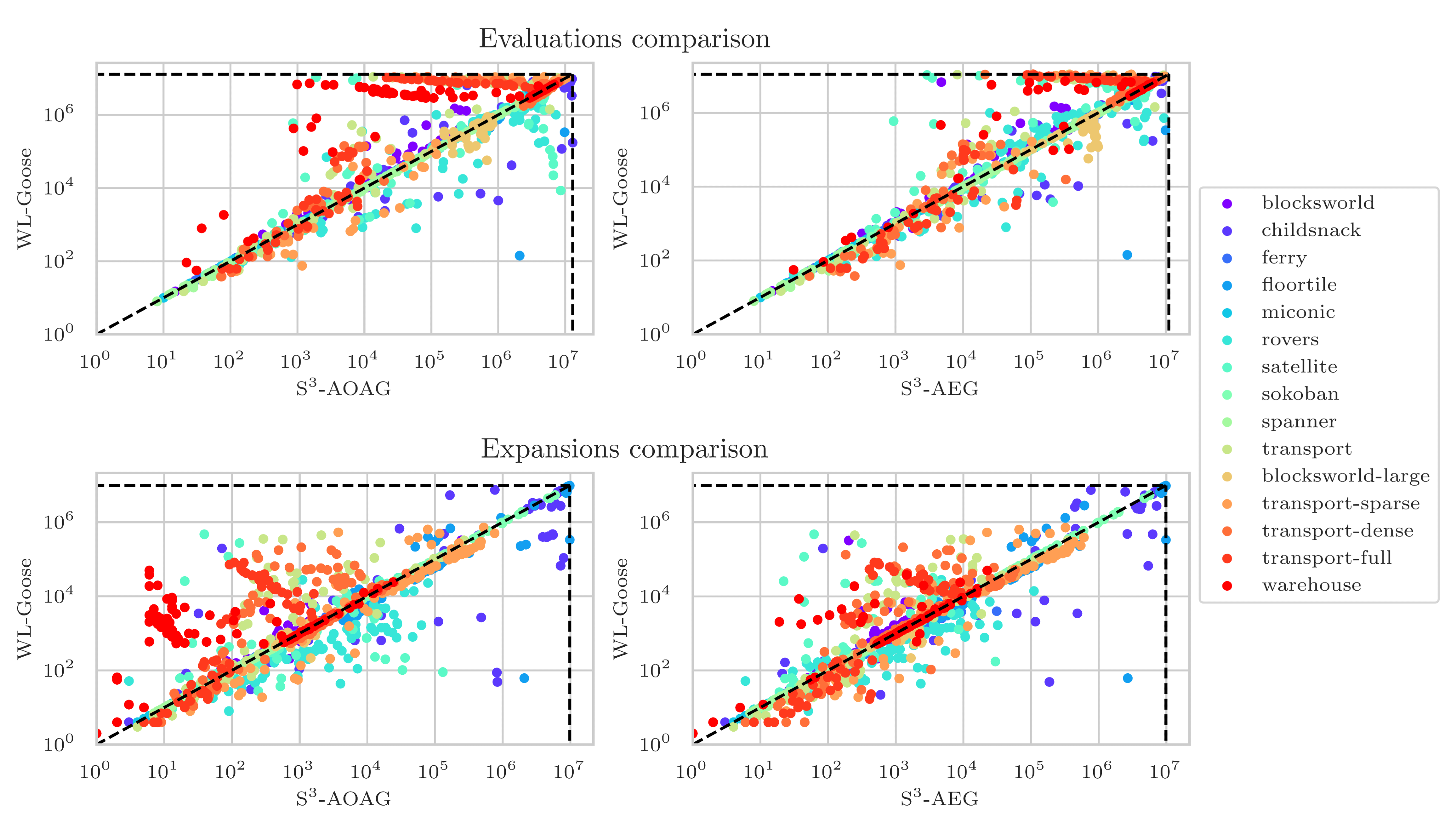}
    \caption{Comparison of learned action set heuristics with state-space search
        versus WL-GOOSE on evaluations (top), and expansions (bottom). The $x=y$
        line is shown in the diagonals. Points above it favour the learned
        action set heuristics and below it favour WL-GOOSE.}
    \label{fig:heuristic-comparisons}
\end{figure*}

\begin{figure*}[!ht]
    \centering
    \includegraphics[width=\textwidth]{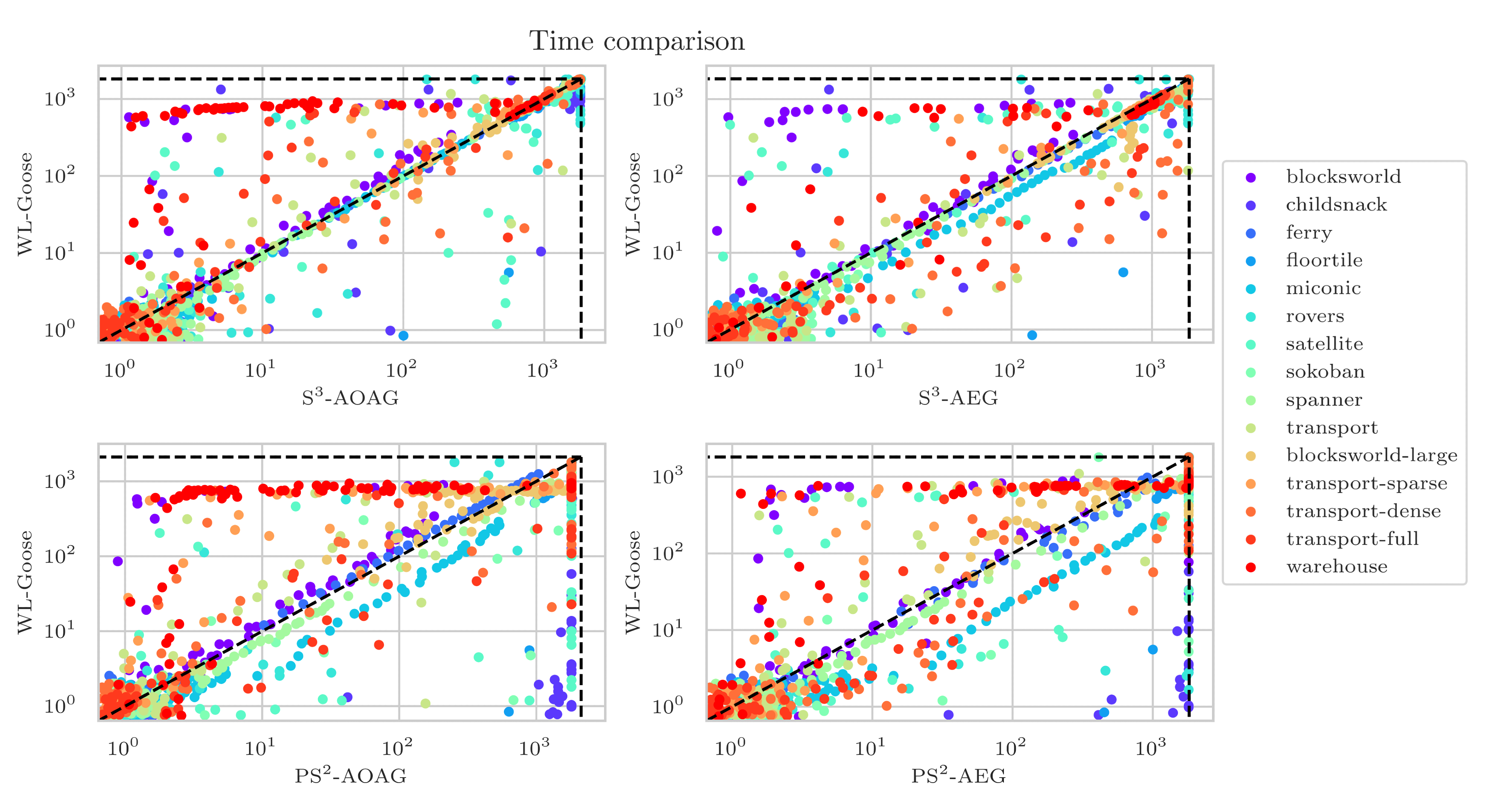}
    \caption{Comparison of learned action set heuristics versus WL-GOOSE on
        runtime. The $x=y$ line is shown in the diagonals. Points above it
        favour the learned action set heuristics and below it favour WL-GOOSE.}
    \label{fig:time-comparison}
\end{figure*}

%

\end{document}